\documentclass[normalheadings]{scrartcl}
\KOMAoptions{
    paper=a4,
    fontsize=11pt,
    cleardoublepage=empty,
    footinclude=true
}
\usepackage{amssymb}
\usepackage{xspace}
\usepackage{color}
\usepackage{subfigure}
\usepackage{graphicx}
\usepackage{amsmath}
\usepackage{amsthm}
\usepackage{amssymb}
\usepackage{float}
\usepackage{url}
\usepackage{tikz}
\usetikzlibrary{arrows,positioning,shapes.misc,shapes.geometric}

\tikzset{
    >=stealth',
    args/.style={circle,draw=black}
}

\newcommand{\powerset}[1]{\ensuremath{2^{#1}}}		
\newcommand{\AF}{\ensuremath{\mathsf{AF}}\xspace}						
\newcommand{\arguments}{\ensuremath{\mathsf{Arg}}\xspace}				
\newcommand{\attacks}{\ensuremath{\rightarrow}\xspace}				
\newcommand{\attackers}[2]{\ensuremath{\mathsf{Att}_{#1}(#2)\xspace}}				
\newcommand{\AFcomplete}{\ensuremath{\AF=(\arguments,\attacks)}\xspace}		
\newcommand{\cA}{\ensuremath{\mathcal{A}}}
\newcommand{\cB}{\ensuremath{\mathcal{B}}}
\newcommand{\cC}{\ensuremath{\mathcal{C}}}

\newcommand{\argin}{\ensuremath{\mathsf{in}}}		
\newcommand{\argout}{\ensuremath{\mathsf{out}}}		
\newcommand{\argundec}{\ensuremath{\mathsf{undec}}}		

\newcommand{\reallyAllProb}{\ensuremath{\mathcal{P}}}
\newcommand{\allProb}[1]{\ensuremath{\mathcal{P}(#1)}}
\newcommand{\allProbT}[2]{\ensuremath{\mathcal{P}_{#1}(#2)}}
\newcommand{\allProbComp}[2]{\ensuremath{\mathcal{P}^{#2}(#1)}}

\newcommand{\dom}[1]{\ensuremath{\mathsf{dom}\,#1}}

\newtheorem{definition}{Definition}
\newtheorem{example}{Example}
\newtheorem{proposition}{Proposition}

\title{Probabilistic Argumentation with Epistemic Extensions and Incomplete Information}
\author{Anthony Hunter$^{\dagger,1}$ \and Matthias Thimm$^{\ddagger,2}$\\~\\[-1ex] \small $^{\dagger}\texttt{anthony.hunter@ucl.ac.uk}$, $^{\ddagger}\texttt{thimm@uni-koblenz.de}$\\ \small$^{1}$Department of Computer Science, University College London, United Kingdom\\[-1ex]\small$^{2}$Institute for Web Science and Technologies (WeST), University of Koblenz-Landau}

\date{\today}

\begin{document}
\maketitle


\begin{abstract}
Abstract argumentation offers an appealing way of representing and evaluating arguments and counterarguments. This approach can be enhanced by a probability assignment to each argument. There are various interpretations that can be ascribed to this assignment. In this paper, we regard the assignment as denoting the belief that an agent has that an argument is justifiable, i.\,e., that both the premises of the argument and the derivation of the claim of the argument from its premises are valid. This leads to the notion of an epistemic extension which is the subset of the arguments in the graph that are believed to some degree (which we defined as the arguments that have a probability assignment greater than 0.5). We consider various constraints on the probability assignment. Some constraints correspond to standard notions of extensions, such as grounded or stable extensions, and some constraints give us new kinds of extensions. 
\end{abstract}

\section{Introduction}
Abstract argumentation, as proposed by Dung \cite{Dung:1995}, provides a simple and appealing representation in the form of a directed graph. Each node denotes an argument and each arc denotes one argument attacking another. Abstract argumentation also provides a set of options for determining which arguments can be accepted together (i.\,e. an extension), and which arguments can be rejected. 
Recently there has been interest in augmenting abstract argumentation with a probabilistic assignment to each argument \cite{DT10,LON11,Hun12b,Thimm:2012}. Once we introduce a probability assignment to each argument, we have extra information about the argumentation, and this means we can refine the evaluation of an argument graph. In most of these proposals (i.\,e. \cite{DT10,LON11,Hun12b}), the emphasis is on what should the structure of the graph be. So the probability of the argument denotes the degree to which the argument should be in the graph. 

In this paper, we take a different approach. We regard the assignment as denoting the belief that an agent has that an argument is justifiable, i.\,e., that both the premises of the argument and the derivation of the claim of the argument from its premises are valid.  So for a probability function $P$, and an argument $\cA$, 
$P(\cA) > 0.5$ denotes that the argument is believed (to the degree given by $P(\cA)$), $P(\cA) < 0.5$ denotes that the argument is disbelieved (to the degree given by $P(\cA)$), and $P(\cA) = 0.5$ denotes that the argument is neither believed or disbelieved. 
This approach leads to the notion of an epistemic extension: This is the subset of the arguments in  the graph that are believed to some degree (i.\,e. the arguments such that $P(\cA) > 0.5$). Since this is a very general idea, our aim in this paper is to consider various properties (i.\,e. constraints) that hold for classes of probability functions, and for the resulting epistemic extensions. We structure our presentation on two views as follows:
\begin{description}
	\item[Standard view] on using probability of arguments.  In this view, we provide properties for the probability function that ensure that the epistemic extensions coincide with Dung's definitions for extensions. Key properties include \emph{coherence} (if $\cA$ attacks $\cB$, then $P(\cA) \leq 1 - P(\cB)$) and \emph{foundation} (if $\cA$ has no attackers, then $P(\cA) = 1$). The advantage of using a probability function instead of Dung's definitions is that we can also specify the degree to which each argument is believed.
	\item[Non-standard view] on using probability of arguments. In this view, we consider alternative properties for the probability function. This means that the resulting epistemic extensions may not coincide with Dung's definitions for extensions.
\end{description}

The framework that we present in this paper is appealing theoretically as it provides further insights into semantics for abstract argumentation, and it offers a finer-grained representation of uncertainty in arguments. Perhaps more importantly, our framework for probability functions is appealing practically because we can now model how audiences judge argumentation. Consider for example how a member of the audience of a debate hears arguments and counterarguments, but is unable (or does not want) to express arguments. Here it is natural to consider how that member of the audience considers which arguments she believes, thereby constructing an epistemic extension. More generally, if we want to make computational models of argument where we can capture persuasion, we need to take account of the belief that participants or audiences have in the individual arguments that are posited.  We see the non-standard view as being particularly important in addressing this need.

This paper builds on previous work \cite{Thimm:2012,Hunter:2013} but extends it in several directions. In particular, the contributions of this paper are as follows:
\begin{enumerate}
	\item We investigate the notion of epistemic extensions and show their usefulness with respect to classical extensions (Section~\ref{section:epistemicextensions}).
	\item We adopt and significantly extend properties for \emph{standard epistemic extensions} from \cite{Thimm:2012,Hunter:2013} and show that these probabilistic concepts coincide with their corresponding concepts from abstract argumentation (Section~\ref{section:standard}).
	\item We introduce non-standard epistemic extensions and a corresponding set of properties as a means to extend the standard view and provide a complete picture of the relationships between our different probabilistic properties (Section~\ref{section:nonstandard}).
	\item We apply the notion of epistemic extensions to the problem of completing partial probability assignments and show the feasibility of this approach (Section~\ref{section:partial}).
\end{enumerate}

\section{Preliminaries}
\label{section:preliminaries}

An \emph{abstract argumentation framework} $\AF$ is a tuple $\AF=(\arguments,\attacks)$ 
where \arguments is a set of arguments and \attacks is a relation $\attacks\subseteq \arguments\times\arguments$. For two arguments $\cA,\cB\in\arguments$ the relation $\cA \attacks \cB$ means that argument $\cA$ attacks argument $\cB$. For $\cA\in\arguments$ define $\attackers{\AF}{\cA}=\{\cB\mid \cB\attacks \cA\}$.

Semantics are given to abstract argumentation frameworks by means of extensions \cite{Dung:1995} or labellings \cite{Wu:2010}. In this work, we use the latter. A labelling $L$ is a function $L:\arguments\rightarrow \{\argin,\argout,\argundec\}$ that assigns to each argument $\cA\in\arguments$ either the value \argin, meaning that the argument is accepted, \argout, meaning that the argument is not accepted, or \argundec, meaning that the status of the argument is undecided. Let $\argin(L)=\{\cA\mid L(\cA)=\argin\}$ and $\argout(L)$ resp.\ $\argundec(L)$ be defined analogously. The set $\argin(L)$ for a labelling $L$ is also called \emph{extension}. A labelling $L$ is called \emph{conflict-free} if for no $\cA,\cB\in\argin(L)$ we have that $\cA\rightarrow \cB$.

Arguably, the most important property of a semantics is its admissibility. A labelling $L$ is called \emph{admissible} if and only if for all arguments $\cA\in\arguments$
\begin{enumerate}
	\item if $L(\cA)=\argout$ then there is $\cB\in\arguments$ with $L(\cB)=\argin$ and $\cB\attacks\cA$, and
	\item if $L(\cA)=\argin$ then $L(\cB)=\argout$ for all $\cB\in\arguments$ with $\cB\attacks \cA$,
\end{enumerate}
and it is called \emph{complete} if, additionally, it satisfies
\begin{enumerate}
	\setcounter{enumi}{2}
	\item if $L(\cA)=\argundec$ then there is no $\cB\in\arguments$ with $\cB\attacks\cA$ and $L(\cB)=\argin$ and there is a $\cB'\in\arguments$ with $\cB'\attacks\cA$ and $L(\cB')\neq\argout$.
\end{enumerate}
The intuition behind admissibility is that an argument can only be accepted if there are no attackers that are accepted and if an argument is not accepted then there has to be some reasonable grounds. The idea behind the completeness property is that the status of an argument is only \argundec\ if it cannot be classified as $\argin$ or $\argout$. Different types of classical semantics can be phrased by imposing further constraints. Let \AFcomplete be an abstract argumentation framework and $L:\arguments\rightarrow \{\argin,\argout,\argundec\}$ a complete labelling. Then
\begin{itemize}
	\item $L$ is \emph{grounded} if and only if $\argin(L)$ is minimal,
	\item $L$ is \emph{preferred} if and only if $\argin(L)$ is maximal,
	\item $L$ is \emph{stable} if and only if $\argundec(L)=\emptyset$, and 
	\item $L$ is \emph{semi-stable} if and only if $\argundec(L)$ is minimal.
\end{itemize}
All statements on minimality/maximality are meant to be with respect to set inclusion.
Note that a grounded labelling is uniquely determined and always exists \cite{Dung:1995}.

\section{Epistemic extensions}\label{section:epistemicextensions}
We now go beyond classical three-valued semantics of abstract argumentation and turn to probabilistic interpretations of the status of arguments. Let $\powerset{\mathcal{X}}$ denote the power set of a set $\mathcal{X}$. A \emph{probability function} $P$ on some finite set $\mathcal{X}$ is a function $P: \powerset{\mathcal{X}}\rightarrow [0,1]$ with $\sum_{X\subseteq \mathcal{X}}P(X) = 1$. Let $P:\powerset{\arguments}\rightarrow [0,1]$ be a probability function on $\arguments$. We abbreviate
\begin{align*}
	P(\cA) & = \sum_{\cA\in E \subseteq \arguments} P(E)
\end{align*}
This means that the probability of an argument is the sum of the probabilities of all sets of arguments that contain that argument.
The following definition is a generalization of the notion of epistemic extensions given in \cite{Hunter:2013}. 
For an argument $\cA$,  
it is labelled $\argin$ when it is believed to some degree (which we identify as $P(\cA) > 0.5$), 
it is labelled $\argout$ when it is disbelieved to some degree (which we identify as $P(\cA) < 0.5$), 
and it is labelled $\argundec$ when it is neither believed nor disbelieved (which we identify as $P(\cA) = 0.5$). More specifically, let $\AFcomplete$ be an abstract argumentation framework and $P:\powerset{\arguments}\rightarrow [0,1]$ a probability function on $\arguments$. The labelling $L_{P}:\arguments\rightarrow \{\argin,\argout,\argundec\}$ defined via the following constraints is called the \emph{epistemic labelling} of $P$:
\begin{itemize}
\item $L_{P}(\cA) = \argin \quad \text{iff}\quad P(\cA) > 0.5$
\item $L_{P}(\cA) = \argout \quad \text{iff}\quad P(\cA) < 0.5$
\item $L_{P}(\cA) = \argundec \quad \text{iff}\quad P(\cA) = 0.5$
\end{itemize}
The \emph{epistemic extension} of $P$ is the set of arguments that are labelled $\argin$ by the epistemic labelling, i.\,e. $X$ is an epistemic extension iff $X=\argin(L_{P})$.
Furthermore, we say that a labelling $L$ and a probability function $P$ are \emph{congruent}, denoted by $L\sim P$, if for all $\cA\in\arguments$ we have $L(\cA)=\argin\Leftrightarrow P(\cA)=1$, $L(\cA)=\argout\Leftrightarrow P(\cA)=0$, and $L(\cA)=\argundec\Leftrightarrow P(\cA)=0.5$. Note, that if $L\sim  P$ then $L=L_{P}$, i.\,e., if a labelling $L$ and a probability function $P$ are congruent then $L$ is also the epistemic labelling of $P$.

An epistemic labelling can be used to give either a standard semantics (as we will investigate in Section \ref{section:standard}) or a non-standard semantics (as we will investigate in Section \ref{section:nonstandard}). 

To further illustrate the epistemic extensions, consider the graph given in Figure \ref{fig:3cycle}.  Here, we may believe that, say, $\cA$ is valid and that $\cB$ and $\cC$ are not valid. In which case, with this extra epistemic information about the arguments, we can resolve the conflict and so take the set $\{ \cA \}$ as the ``epistemic" extension. In contrast, there is only one admissible set which is the empty set. So by having this extra epistemic information, we get a more informed extension (in the sense that it has harnessed the extra information in a sensible way).
\begin{figure}
\begin{center}
			\begin{tikzpicture}[->,>=latex]
			\tikzstyle{every node}=[font=\scriptsize]
			\node[text centered,text width=3cm,draw] (a) at (6,3) {$\cA$ = Ann will go to the party and this means that Bob will not go to the party};
			\node[text centered,text width=3cm,draw] (b) at (4,0) {$\cB$ = Bob will go to the party and this means that Chris will not go to the party};
			\node[text centered,text width=3cm,draw] (c) at (8,0) {$\cC$ = Chris will go to the party and this means that Ann will not go to the party};
			\path (b) edge (c);
			\path (c) edge (a);
			\path (a) edge (b);
			\end{tikzpicture}
\end{center}
\caption{\label{fig:3cycle}Example of three arguments in a simple cycle.}
\end{figure}
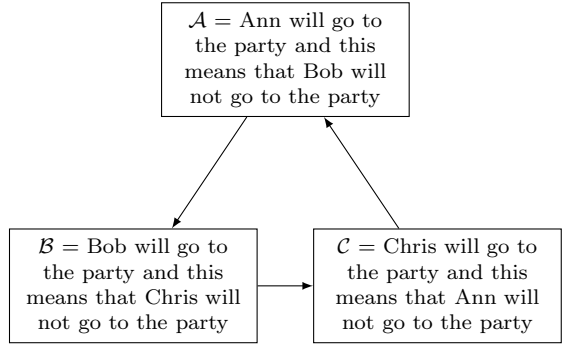
		
In general, we want epistemic extensions to allow us to better model the audience of argumentation. Consider, for example, when a member of the audience of a TV debate listens to the debate at home, she can produce the abstract argument graph based on the arguments and counterarguments exchanged. Then she can identify a probability function to represent the belief she has in each of the arguments. So she may disbelieve some of the arguments based on what she knows about the topic. Furthermore, she may disbelieve some of the arguments that are unattacked. 
As an extreme, she is at liberty of completely disbelieving all of the arguments (so to assign probability 0 to all of them). If we want to model audiences, where the audience either does not want to or is unable to add counterarguments to an argument graph being constructed in some form of argumentation, we need to take the beliefs of the audience into account, and we need to consider which arguments they believe or disbelieve.

\section{Standard epistemic extensions}
\label{section:standard}

We now consider some constraints on the probability function which may take different aspects of the structure of the argument graph into account. We will show how these constraints are consistent with Dung's notions of admissibility.

For the remainder of this paper let $\AFcomplete$ be an abstract argumentation framework and $P:\powerset{\arguments}\rightarrow [0,1]$. Consider the following properties (note that \textsf{COH} is from \cite{Hunter:2013} and \textsf{JUS} is from \cite{Thimm:2012}):
\begin{description}
	\item[COH] $P$ is \emph{coherent} wrt. $\AF$ if for every $\cA,\cB\in\arguments$, if $\cA \attacks \cB$ then $P(\cA)\leq 1-P(\cB)$.
	\item[SFOU] $P$ is \emph{semi-founded} wrt. $\AF$ if $P(\cA)\geq 0.5$ for every $\cA\in\arguments$ with $\attackers{\AF}{\cA}=\emptyset$.
	\item[FOU] $P$ is \emph{founded} wrt. $\AF$ if $P(\cA)=1$ for every $\cA\in\arguments$ with $\attackers{\AF}{\cA}=\emptyset$.
	\item[SOPT] $P$ is \emph{semi-optimistic} wrt. $\AF$ if $P(\cA)\geq 1- \sum_{\cB\in\attackers{\AF}{\cA}}P(\cB)$ for every $\cA\in\arguments$ with $\attackers{\AF}{\cA}\neq\emptyset$.
	\item[OPT] $P$ is \emph{optimistic} wrt. $\AF$ if $P(\cA)\geq 1- \sum_{\cB\in\attackers{\AF}{\cA}}P(\cB)$ for every $\cA\in\arguments$.
	\item[JUS] $P$ is \emph{justifiable} wrt. $\AF$ if $P$ is coherent and optimistic.
	\item[TER] $P$ is \emph{ternary} wrt. $\AF$ if $P(\cA) \in \{0,0.5,1\}$ for every $\cA\in\arguments$.
\end{description} 
We explain these constraints as follows:
\textsf{COH} ensures that if argument $\cA$ attacks argument $\cB$, then the degree to which $\cA$ is believed caps the degree to which $\cB$ can be believed;
\textsf{SFOU} ensures that if an argument is not attacked, then the argument is not disbelieved (i.\,e. $P(\cA) \geq 0.5$);
\textsf{FOU} ensures that if an argument is not attacked, then the argument is believed without doubt (i.\,e. $P(\cA) = 1$);
\textsf{SOPT} ensures that the belief in $\cA$ is bounded from below if the belief in its attackers is not high;
\textsf{OPT} ensures that if an argument is not attacked, then the argument is believed without doubt (i.\,e. $P(\cA) = 1$) and that the belief in $\cA$ is bounded from below if the belief in its attackers is not high;
\textsf{JUS} combines \textsf{COH} and \textsf{OPT} to provide bounds on the belief in an argument based on the belief in its attackers and attackees;
and \textsf{TER} ensures that the probability assignment is a three-valued assignment.
\begin{example}\label{Example:AF1}
	Consider the abstract argumentation framework $\AF=(\arguments,\attacks)$ depicted in Fig.~\ref{fig:AF1} and the probability functions depicted in Table~\ref{table:probfunctions1} (note that the probability functions there are only partially defined by giving the probabilities of arguments). The following observations can be made:
1.) $P_{1}$ is semi-founded, founded, but neither coherent, optimistic, semi-optimistic, ternary, nor justifiable,
2.) $P_{2}$ is coherent and semi-optimistic, but neither semi-founded, founded, optimistic, ternary, nor justifiable,
3.) $P_{3}$ is coherent, semi-optimistic, semi-founded, founded, optimistic, and justifiable, but not ternary,
4.) $P_{4}$ is semi-founded, founded, optimistic, and semi-optimistic, but neither coherent, justifiable, nor ternary, and
5.) $P_{5}$ is coherent, semi-founded, semi-optimistic, and ternary but neither optimistic, justifiable, nor founded.
	\begin{table}
		\begin{center}
		\begin{tabular}{|c|c|c|c|c|c|c|}
		\hline
			&	$\cA_{1}$&	$\cA_{2}$&	$\cA_{3}$&	$\cA_{4}$&	$\cA_{5}$ & $\cA_{6}$\\
			\hline\hline
			$P_{1}$ & 0.2	&	0.7	&	0.6	& 0.3 & 0.6 & 1\\	\hline
			$P_{2}$ & 0.7	& 	0.3	&	0.5	& 0.5 & 0.2 & 0.4\\\hline
			$P_{3}$ & 0.7	&	0.3	&	0.7  & 0.3  & 0 & 1\\\hline		
			$P_{4}$ & 0.7	&	0.8	&	0.9  & 0.8  & 0.7 & 1\\\hline		
			$P_{5}$ & 0.5	&	0.5	&	0.5  & 0.5 & 0.5 & 0.5\\\hline
		\end{tabular}
		\end{center}		
		
		\caption{Some probability functions for Example~\ref{Example:AF1}}
		\label{table:probfunctions1}
	\end{table}	
\end{example}
	\begin{figure}
		\begin{center}
			\begin{tikzpicture}[node distance=0.7cm]
			
				\node[args](args1){$\cA_{1}$};
				\node[args, right=of args1](args2){$\cA_{2}$};
				\node[args, right=of args2](args3){$\cA_{3}$};
				\node[args, right=of args3, yshift=0.7cm](args4){$\cA_{4}$};
				\node[args, right=of args3, yshift=-0.7cm](args5){$\cA_{5}$};
				\node[args, right=of args5](args6){$\cA_{6}$};
			
				\path(args1) edge [->,bend left] (args2);
				\path(args2) edge [->,bend left] (args1);
				\path(args2) edge [->] (args3);
				\path(args3) edge [->] (args4);
				\path(args4) edge [->, bend left] (args5);
				\path(args5) edge [->, bend left] (args4);
				\path(args5) edge [->] (args3);
				\path(args6) edge [->] (args5);

			\end{tikzpicture}
		\end{center}		
		\caption{A simple argumentation framework}
		\label{fig:AF1}
	\end{figure}
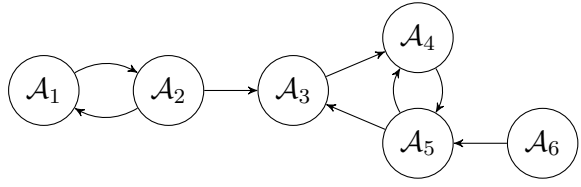
Let $\reallyAllProb$ be the set of all probability functions, $\allProb{\AF}$ be the set of all probability functions on $\arguments$, and $\allProbT{t}{\AF}$ be the set of all $t$-probability functions with $t\in\{$\textsf{COH},\textsf{SFOU},\textsf{FOU}, \textsf{OPT},\textsf{SOPT},\textsf{JUS},\textsf{TER}$\}$. We obtain the following relationships between the different classes of probability functions.
\begin{proposition}\label{prop:hasse1}
	Let $\AFcomplete$ be an abstract argumentation framework.
	\begin{enumerate}
		\item $\emptyset \subsetneq\allProbT{\mathsf{JUS}}{\AF}\subsetneq \allProbT{\mathsf{COH}}{\AF}\subsetneq\allProb{\AF}$
		\item $\allProbT{\mathsf{OPT}}{\AF} = \allProbT{\mathsf{SOPT}}{\AF}\cap \allProbT{\mathsf{FOU}}{\AF}$.
		\item $\allProbT{\mathsf{FOU}}{\AF} \subsetneq \allProbT{\mathsf{SFOU}}{\AF}$.
		\item $\emptyset \subsetneq\allProbT{\mathsf{TER}}{\AF} \subsetneq \allProb{\AF}$.
\end{enumerate}
\end{proposition}
For the proof of item 1.) of the above proposition see \cite{Thimm:2012} and \cite{Hunter:2013}. The remaining relationships follow directly from these definitions.

For all probability functions $P$ such that $L_P$ is admissible in the classical sense, we have that $P$ assigns some degree of belief to each argument that is unattacked, thereby $P$ satisfies the \textsf{SFOU} constraint.
\begin{proposition}
	For all probability functions $P$, if $L_P$ is admissible then $P \in \allProbT{\mathsf{SFOU}}{\AF}$. 
	\begin{proof}
	Assume $L_P$ is admissible.
	Therefore, if $L_P(\cA) = \argout$, then there is an argument $\cB$ such that $\cB \rightarrow \cA$ and $L_P(\cA) = \argin$.
	Therefore, if $\attackers{\AF}{\cA}=\emptyset$, then  $L_P(\cA) \neq \argout$.
	Therefore, if $\attackers{\AF}{\cA}=\emptyset$, then $P(\cA) \geq 0.5$.
	Therefore, $P \in \allProbT{\mathsf{SFOU}}{\AF}$.
	\end{proof}
\end{proposition}
We can further constrain a probability assignment, so that the epistemic labelling straightforwardly captures the standard semantics (i.\,e. Dung's semantics). By setting the probability function appropriately, its epistemic labelling corresponds to grounded, complete, stable, preferred, or semi-stable labellings. All we require is a three-valued probability function that simulates each complete labelling function. For this, we provide the following definition that provides the counter-part in our framework for a complete labelling.
\begin{definition}\label{def:completeprobfunc}
	Let $\AFcomplete$ be an argumentation framework. Then a \emph{complete probability function} $P\in\allProb{\AF}$ for $\AF$ is a probability function $P$
such that for every $\cA \in \arguments$ the following conditions hold:
	\begin{enumerate}
		\item $P \in \allProbT{\mathsf{TER}}{\AF}$; 
		\item if $P(\cA)=1$ then $P(\cB)=0$ for all $\cB\in\arguments$ with $\cB\attacks \cA$;
		\item if $P(\cB)=0$ for all $\cB$ with $\cB\attacks\cA$ then $P(\cA) = 1$; 		
		\item if $P(\cA)=0$ then there is $\cB\in\arguments$ with $P(\cB)=1$ and $\cB\attacks\cA$;		
		\item if $P(\cB)=1$ for some $\cB$ with $\cB\attacks \cA$ then $P(\cA)=0$.
	\end{enumerate}
\end{definition}
Note that the above definition straightforwardly follows the definition of completeness for classical semantics. Therefore, we have that $P$ is a complete probability function if and only if there is a complete labeling $L$ and $P\sim L$.

Completeness of probability functions can be characterized by the aforementioned properties as follows.
\begin{proposition}
	For an argument graph $\AF$, and for $P \in \allProbT{}{\AF}$, $P$ is a complete probability function iff $P \in \allProbT{\mathsf{COH}}{\AF} \cap \allProbT{\mathsf{FOU}}{\AF} \cap \allProbT{\mathsf{TER}}{\AF}$.
\begin{proof}~
($\Rightarrow$)
Assume $P$ is a complete probability function.
By Definition 1 Part 1, 
$P \in \allProbT{\mathsf{TER}}{\AF}$.
By Definition 1 Part 3, 
if $\attackers{\AF}{\cA}=\emptyset$, 
then $P(\cA) = 1$,
and hence $P \in \allProbT{\mathsf{FOU}}{\AF}$.
By Definition 1 Parts 2 to 5,
\begin{itemize}
\item $P(\cA) = 1$ iff for all $\cB$ s.t. $\cB \rightarrow \cA$, $P(\cB) = 0$
\item $P(\cA) = 0.5$ iff there is not an $\cB$ s.t. $\cB \rightarrow \cA$, $P(\cB) = 1$ and there is a $\cB'$ s.t. $\cB' \rightarrow \cA$, $P(\cB') = 0.5$
\item $P(\cA) = 0$ iff there is a $\cB$ s.t. $\cB \rightarrow \cA$, $P(\cB) = 1$
\end{itemize}
So for all attacks $\cB \rightarrow \cA$, $P(\cA) \leq 1 -P(\cB)$. 
Hence, $P \in \allProbT{\mathsf{COH}}{\AF}$. 
($\Leftarrow$)
Assume $P \in \allProbT{\mathsf{COH}}{\AF} \cap \allProbT{\mathsf{FOU}}{\AF} \cap \allProbT{\mathsf{TER}}{\AF}$.
We now show that Parts 1 to 5 of Definition 1 are satisfied for $P$.
From $P \in \allProbT{\mathsf{TER}}{\AF}$, $P$ satisfies Part 1.
From $P \in \allProbT{\mathsf{COH}}{\AF} \cap \allProbT{\mathsf{FOU}}{\AF}$, $P$ satisfies Parts 2 and 3.
From $P \in \allProbT{\mathsf{COH}}{\AF}$, $P$ satisfies Parts 4 and 5.\qedhere
	\end{proof}
\end{proposition}
In the same way that Caminada and Gabbay \cite{Caminada:2009} showed that different semantics can be obtained by imposing further restrictions on the choice of labelling, we can obtain the different semantics by imposing further restrictions on the choice of complete probability function. These constraints, as shown in the following result, involve minimizing or maximizing particular assignments. So for instance, if the assignment of $1$ to arguments is maximized, then a preferred labelling is obtained.

\begin{proposition}\label{prop:standard0}
Let $\AFcomplete$ be an abstract argumentation framework and $P \in \allProb{\AF}$.
If $P$ is a complete probability function for $\AF$
and the restriction specified in Table~\ref{tbl:correspondences} holds for $P$,
then the corresponding type of epistemic labelling is obtained.
\begin{proof}
Let $L$ and $P$ be congruent.
So $L$ is a complete labelling iff $P$ is a complete probability assignment.
Therefore, each restriction in Section \ref{section:preliminaries} holds for $L$ 
iff the corresponding restriction in Table~\ref{tbl:correspondences} holds for $P$.
For instance, ``Maximal number of arguments $\cA$ such that $L(\cA) = \argundec$" holds 
iff ``Maximal number of arguments $\cA$ such that $P(\cA) = 0.5$" holds.
Therefore, the corresponding type of extension for the restriction on $L$ (as listed in Section \ref{section:preliminaries}
and proven to hold by Caminada and Gabbay in \cite{Caminada:2009}),  
also hold for the equivalent restriction on $P$ in the Table~\ref{tbl:correspondences}.
\end{proof}
\end{proposition}
\begin{table}
\begin{center}
\begin{tabular}{|c|c|}
\hline
Restriction on probability function $P$ & Classical semantics\\
\hline
\hline
No restriction & complete extensions\\
No arguments $\cA$ such that $P(\cA) = 0.5$ & stable\\
Maximal no. of $\cA$ such that $P(\cA) = 1$ & preferred \\
Maximal no. of $\cA$ such that $P(\cA) = 0$ & preferred\\
Maximal no. of $\cA$ such that $P(\cA) = 0.5$ & grounded\\
Minimal no. of $\cA$ such that $P(\cA) = 1$ & grounded\\
Minimal no. of $\cA$ such that $P(\cA) = 0$ & grounded\\
Minimal no. of $\cA$ such that $P(\cA) = 0.5$ & semi-stable\\
\hline
\end{tabular}
\caption{Correspondences between probabilistic and classical semantics}
\label{tbl:correspondences}
\end{center}
\end{table}
For an argumentation framework $\AF$ we can identify specific probability functions in $P \in \allProbT{\mathsf{JUS}}{\AF}$ that are congruent with admissible labellings, grounded labellings, or stable labellings, for $\AF$ as follows.
\begin{proposition}\label{prop:standard1}(From \cite{Thimm:2012})
	Let $\AFcomplete$ be an abstract argumentation framework.
\begin{enumerate}
		\item For every admissible $L$ there is $P\in\allProbT{\mathsf{JUS}}{\AF}$ with $L\sim P$.
		\item Let $L$ be the grounded labelling and let\footnote{Define the \emph{entropy} $H(P)$ of $P$ as $H(P)=-\sum_{E\subseteq \arguments}P(E)\log P(E)$} $P=\arg\max_{Q\in \allProbT{\mathsf{JUS}}{\arguments}}H(Q)$. Then $L\sim P$.
		\item Let stable labellings exist for $\AF$ and let $L$ be a stable labelling. Then there is $P\in \arg\min_{Q\in \allProbT{\mathsf{JUS}}{\arguments}}H(Q)$ with $L\sim P$.
	\end{enumerate}
\end{proposition}
So Proposition \ref{prop:standard0} and Proposition \ref{prop:standard1} provide two ways to identify probability functions that capture specific types of labellings. Each of these results show that standard notions of classical semantics (i.\,e.\ admissibility and the definitions for different kinds of labelling such as grounded labellings, stable labellings, etc.) can be captured using probability functions. 

The next result shows that using probability functions to capture labellings gives a finer-grained formalization of classical semantics.

\begin{proposition}
	For each complete labelling $L$, if there is an argument $\cA$ such that $L(\cA) \neq \argundec$, then there are infinitely many probability functions $P$ such that $L_P= L$. 
	\begin{proof}
		Let $L$ be a complete labeling such that there is an argument $\cA$ with $L(\cA) \neq \argundec$. Without loss of generality assume that $L(\cA)=\argin$. Then every probability function with $P(\cB)=0.5$ iff $L(\cB)=\argundec$, $P(\cB)=0$ iff $L(\cB)=\argout$, $P(\cB)=1$ iff $L(\cB)=\argin$ and $\cB\neq \cA$, and $P(\cA)\in(0.5,1]$ yields $L_{P}=L$.
	\end{proof}
\end{proposition}
Obviously, for every probability function $P$, there is by definition exactly one epistemic labelling $L_P$. This means that using a probability function to identify which arguments are $\argin$, $\argundec$, or $\argout$, subsumes using labels. Furthermore, the probability function captures more information about the arguments. The granularity can differentiate between for example a situation where $\cA$ is believed (i.\,e. it is $\argin$) with certainty by $P(\cA) =1$ from a situation where $\cA$ is only just believed (i.\,e. it is only just $\argin$) for example by $P(\cA) =0.51$. Similarly, we can differentiate a situation where an attack by $\cB$ on $\cA$ is undoubted when $P(\cB) = 1$ and $P(\cA) = 0$ from a situation where an attack by $\cB$ on $\cA$ is very much doubted when for example $P(\cB) = 0.55$ and $P(\cA) = 0.45$.

In conclusion, we have shown how axioms can be used to constrain the probability function, and thereby constrain the epistemic labelings and the epistemic extensions. This allows us to subsume Dung's notions of extensions as epistemic extensions. Furthermore, we get a finer-grained representation of the labelling of arguments by representing the belief in each of the arguments.

\section{Non-standard epistemic extensions}\label{section:nonstandard}
Before exploring the notion of non-standard epistemic extensions, we will augment the set of properties we introduced in the previous section with the following properties. 
Let $\AFcomplete$ be an abstract argumentation framework and $P:\powerset{\arguments}\rightarrow [0,1]$. 
\begin{description}
	\item[RAT] $P$ is \emph{rational} wrt. $\AF$ if for every $\cA,\cB\in\arguments$, if $\cA \attacks \cB$ then $P(\cA)>0.5$ implies $P(\cB)\leq 0.5$.
	\item[NEU] $P$ is \emph{neutral} wrt. $\AF$ if $P(\cA) = 0.5$ for every $\cA\in \arguments$.
	\item[INV] $P$ is \emph{involutary} wrt. $\AF$ if for every $\cA,\cB \in \arguments$, if $\cA \attacks \cB$, then $P(\cA) = 1 - P(\cB)$.
	\item[MAX] $P$ is \emph{maximal} wrt. $\AF$ if $P(\cA) = 1$ for every $\cA\in \arguments$.
	\item[MIN] $P$ is \emph{minimal} wrt. $\AF$ if $P(\cA) = 0$ for every $\cA\in \arguments$.	
\end{description} 
We explain these constraints as follows:
\textsf{RAT} ensures that if argument $\cA$ attacks argument $\cB$, and $\cA$ is believed (i.\,e. $P(\cA) > 0.5$), then $\cB$ is not believed (i.\,e. $P(\cB) \leq 0.5$);
\textsf{NEU} ensures that all arguments are neither believed nor disbelieved (i.\,e. $P(\cA) = 0.5$ for all arguments);
\textsf{INV} ensures that if argument $\cA$ attacks argument $\cB$, then the belief in $\cA$ is the inverse of the belief in $\cB$; 
\textsf{MAX} ensures that all arguments are completely believed;
and
\textsf{MIN} ensures that all arguments are completely disbelieved.
\begin{example}\label{Example:AF2}
	We continue Example~\ref{Example:AF1}, the abstract argumentation framework $\AF=(\arguments,\attacks)$ depicted in Fig.~\ref{fig:AF1}, and the probability functions depicted in Table~\ref{table:probfunctions1}. The following observations can be made:
1.) $P_{2}$ and $P_{3}$ are rational but neither neutral, involutary, maximal, nor minimal,	
2.) $P_{1}$ and $P_{4}$ are neither rational, neutral, involutary, maximal, nor minimal, and
3.) $P_{5}$ is rational, neutral, and involutary but neither maximal nor minimal.
\end{example}	
As before let $\allProbT{t}{\AF}$ be the set of all $t$-probability functions with $t\in\{$\textsf{COH},\textsf{SFOU}, \textsf{FOU}, \textsf{SOPT}, \textsf{OPT}, \textsf{JUS}, \textsf{TER}, \textsf{RAT}, \textsf{NEU}, \textsf{INV}, \textsf{MAX}, \textsf{MIN}$\}$. We extend the classification from Proposition~\ref{prop:hasse1} as follows.
\begin{proposition}\label{prop:hasse2}
	Let $\AFcomplete$ be an abstract argumentation framework.
	\begin{enumerate}
		\item $\emptyset \subsetneq\allProbT{\mathsf{JUS}}{\AF}\subsetneq \allProbT{\mathsf{COH}}{\AF}\subsetneq\allProbT{\mathsf{RAT}}{\AF}\subsetneq\allProb{\AF}$
		\item $\emptyset \subsetneq \allProbT{\mathsf{NEU}}{\AF} \subseteq \allProbT{\mathsf{INV}}{\AF}  \subsetneq \allProbT{\mathsf{COH}}{\AF}$ 
		\item $\emptyset \subsetneq \allProbT{\mathsf{INV}}{\AF} \subsetneq \allProbT{\mathsf{SOPT}}{\AF}$ 
		\item $\emptyset \subsetneq \allProbT{\mathsf{MIN}}{\AF} \subsetneq \allProbT{\mathsf{COH}}{\AF}$
		\item $\emptyset \subsetneq \allProbT{\mathsf{MAX}}{\AF} \subsetneq \allProbT{\mathsf{OPT}}{\AF}$
	\end{enumerate}
	\begin{proof}
		We only give the proof for 2.). The proofs for 1.) can be found in \cite{Thimm:2012} and \cite{Hunter:2013}, the remaining proofs are straightforward.
		
		The probability function $P$ with $P(E)=1/|2^{\arguments}|$ for all $E\subseteq \arguments$ has $P(\cA)=0.5$ for all $\cA\in\arguments$ and is therefore neutral. It follows that $\allProbT{\mathsf{NEU}}{\AF}\neq \emptyset$ for every $\AF$. Furthermore, if $P\in\allProbT{\mathsf{NEU}}{\AF}$ then for every $\cA,\cB\in\arguments$, if $\cA \attacks \cB$ we have trivially $P(\cA)= 1-P(\cB)$, so $P\in\allProbT{\mathsf{INV}}{\AF}$ and then also $P(\cA)\leq 1-P(\cB)$, i.\,.e, $P\in\allProbT{\mathsf{COH}}{\AF}$. Finally, for $\AFcomplete$ with $\arguments=\{\cA,\cB\}$ and $\attacks=\{(\cA,\cB)\}$ any probability function $P$ with $P(\cA)=0.4$ and $P(\cB)=0.4$ is coherent but not involutary. 
	\end{proof}
\end{proposition}
Together with Examples~\ref{Example:AF1} and \ref{Example:AF2} we obtain the strict classification of classes of probability functions as depicted in Figure~\ref{fig:classification}.
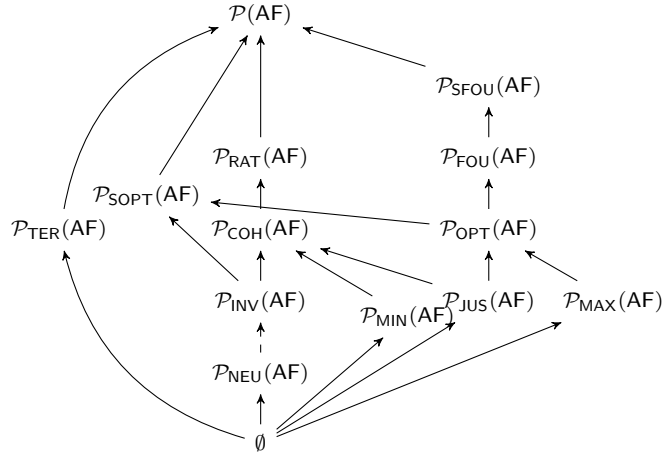
\begin{figure}[!t]
		\begin{center}
		\begin{tikzpicture}[node distance=0.4cm]
			\tikzstyle{every node}=[font=\scriptsize]
			\node(all){$\allProb{\AF}$};
			
			\node[below=of all,xshift=3cm](sfou){$\allProbT{\mathsf{SFOU}}{\AF}$};
			
			\node[below=of sfou,xshift=-3cm](rat){$\allProbT{\mathsf{RAT}}{\AF}$};
			\node[below=of sfou,xshift=-4.5cm,yshift=-0.5cm](sopt){$\allProbT{\mathsf{SOPT}}{\AF}$};
			\node[below=of sfou](fou){$\allProbT{\mathsf{FOU}}{\AF}$};
			
			\node[below=of rat](coh){$\allProbT{\mathsf{COH}}{\AF}$};
			\node[below=of fou](opt){$\allProbT{\mathsf{OPT}}{\AF}$};
			
			\node[below=of opt](jus){$\allProbT{\mathsf{JUS}}{\AF}$};
			\node[below=of coh](inv){$\allProbT{\mathsf{INV}}{\AF}$};
			\node[below=of inv](neu){$\allProbT{\mathsf{NEU}}{\AF}$};

			\node[left=of coh,xshift=-0.7cm](ter){$\allProbT{\mathsf{TER}}{\AF}$};

			\node[right=of neu, yshift=0.75cm](min){$\allProbT{\mathsf{MIN}}{\AF}$};
			\node[right=of jus, xshift=-0.3cm](max){$\allProbT{\mathsf{MAX}}{\AF}$};

			\node[below=of neu](empty){$\emptyset$};
					
			\path(rat) edge [->] (all);
			\path(sopt) edge [->] (all);
			\path(fou) edge [->] (sfou);
			\path(sfou) edge [->] (all);
			\path(coh) edge [->] (rat);
			\path(opt) edge [->] (sopt);
			\path(opt) edge [->] (fou);
			\path(jus) edge [->] (coh);
			\path(jus) edge [->] (opt);
			\path(inv) edge [->] (coh);
			\path(neu) edge [->,dashed] (inv);
			\path(inv) edge [->] (sopt);
			\path(empty) edge [->, bend left] (ter);
			\path(ter) edge [->, bend left] (all);
			\path(empty) edge [->] (jus);
			\path(empty) edge [->] (neu);
			\path(empty) edge [->] (min);
			\path(empty) edge [->] (max);
			\path(max) edge [->] (opt);
			\path(min) edge [->] (coh);
		\end{tikzpicture}
	\end{center}
	\caption{Classes of probability functions (a normal arrow $\rightarrow$ indicates a strict subset relation, a dashed arrow $\dashrightarrow$ indicates a subset relation)}
	\label{fig:classification}
\end{figure}

The \textsf{RAT} constraint is a weaker version of the \textsf{COH} constraint, and it can be used to capture each admissible labelling as a probability function. 

\begin{proposition}
If $L$ is an admissible labelling, then there is a $P \in \allProbT{\mathsf{RAT}}{\AF}$ such that $L \sim P$.
\end{proposition}

Furthermore, the epistemic labelling corresponding to each probability function that satisfies the \textsf{RAT} property is conflict-free.

\begin{proposition}
Let $\AFcomplete$ be an abstract argumentation framework. 
For each $P \in P_{\textsf{RAT}}(AF)$, ${\sf in}(L_P)$ is a conflict-free set of arguments in $\AF$.
\begin{proof}
Assume $P \in \allProbT{\mathsf{RAT}}{\AF}$.
So for each $\cA \rightarrow \cB$, either $P(A) \leq 0.5$ or $P(B) \leq 0.5$.
Therefore, for each $\cA \rightarrow \cB$, if $L_P(\cA) = \argin$, then $L_P(\cB) \neq \argin$,
and if $L_P(\cB) = \argin$, then $L_P(\cA) \neq \argin$.
Hence, ${\sf in}(L_P)$ is a conflictfree set of arguments in $\AF$.
\end{proof}
\end{proposition}

When the argument graph has odd cycles, there is no probability function that is involutary, apart from a neutral probability function. 
\begin{proposition}
Let $\AFcomplete$ be an abstract argumentation framework. 
If $\AF$ contains an odd cycle (i.\,e. there is a sequence of attacks $A_1 \rightarrow A_2 \rightarrow ...... \rightarrow A_k$ where $A_1 = A_k$ and $k$ is an even number), 
and $P \in \allProbT{\mathsf{INV}}{\AF}$
then $P \in P_{\mathsf{NEU}}(\AF)$.
\begin{proof}
Assume that there is a sequence of attacks $A_1 \rightarrow A_2 \rightarrow ...... \rightarrow A_k$ where $A_1 = A_k$ and $k$ is an even number.
Let $P(A_1) = \alpha$.
Hence, $P(A_2) = 1 - \alpha$, $P(A_3) = \alpha$, \ldots, $P(A_{k-1}) = \alpha$, $P(A_{k}) = 1- \alpha$.
Therefore,  $P(A_1) = \alpha$ and $P(A_{k}) = 1- \alpha$.
Yet $A_1 = A_k$.
This is only possible if $\alpha = 0.5$.
Hence, $P \in P_{\mathsf{NEU}}(\AF)$.
\end{proof}
\end{proposition}

Even when the graph is acyclic, it may be the case that there is no involutary probability function (apart from the neutral probability function). Consider for example an argument graph containing three arguments $\cA$, $\cB$ and $\cC$, with $\cA$ attacking both $\cB$ and $\cC$, and $\cB$ attacking $\cC$. For this, there is no involutary probability function  (apart from the neutral probability function). If we restrict consideration to trees, then we are guaranteed to have a probability function that is involutary and not neutral. But even here there are constraints such as siblings have to have the same assignment as captured in the next result.

\begin{proposition}
If $P \in \allProbT{\mathsf{INV}}{\AF}$, then for all $\cB_i,\cB_j \in\attackers{\AF}{\cA}$ we have $P(\cB_i) = P(\cB_j)$.
\begin{proof}
Let $\cB_i \rightarrow \cA$ and $\cB_j \rightarrow \cA$ be attacks.
Assume $P \in \allProbT{\mathsf{INV}}{\AF}$.
Therefore, $P(\cA) = 1 - P(\cB_i)$ and $P(\cA) = 1 - P(\cB_j)$.
Hence, $P(\cB_i) = P(\cB_j)$.
\end{proof}
\end{proposition}

When $P \in \allProbT{\mathsf{MAX}}{\AF}$, the probability function does not take the structure of the graph into account. Hence, there is an incompatibility between a probability function being maximal and a probability function being either rational or coherent (as shown in the proposition below). However, there is compatibility between a probability function being maximal and a probability function being founded since each $P \in \allProbT{\mathsf{MAX}}{\AF}$ is in $\allProbT{\mathsf{FOU}}{\AF}$. 


\begin{proposition}
Let $\AFcomplete$ be an abstract argumentation framework. 
If there are $\cA,\cB\in \arguments$ such that $\cA \rightarrow \cB$, 
then $\allProbT{\mathsf{RAT}}{AF} \cap \allProbT{\mathsf{MAX}}{AF} = \emptyset$.
\end{proposition}

\begin{proof}
Assume there is an attack $\cA \rightarrow \cB$.
So for all $P \in \allProbT{\mathsf{RAT}}{AF}$, 
if $P(\cA) = 1$, then $P(\cB) \leq 0.5$, and $P(\cB) = 1$, then $P(\cA) \leq 0.5$.
And for all $P \in \allProbT{\mathsf{MAX}}{AF} $, $P(\cA) = 1$ and $P(\cB) = 1$.
So $\allProbT{\mathsf{RAT}}{AF} \cap \allProbT{\mathsf{MAX}}{AF} = \emptyset$.
\end{proof}

In conclusion, we have identified epistemic extensions that are obtained from rational probability functions as being an appealing alternative to extensions obtained by Dung's definitions. Rational probability functions are more general than coherent probability functions, and allow the audience more flexibility in expressing their beliefs in the arguments whilst taking the structure of the argument graph into account. We have also considered alternatives such as the involutary probability functions but these are over-constrained.

\section{Application: Partial probability functions}\label{section:partial}

Assigning a probability value to an argument can be useful for a variety of purposes such as representing the belief that the premises of the argument are valid, or the belief that the claim is valid given that the premises are valid, or the belief that both the premises and claim are valid, or the belief that the argument should appear in the argument graph, etc. 
However, given an argument graph, it may be difficult for a user to assign a value to every argument. The user might have knowledge in order to identify a value for some arguments, but the user may be unable or unwilling to make assignments to the remaining arguments. This means that the user can only provide a partial assignment. If this is the case, then it would be desirable to have techniques to handle this incomplete information. Ideally, we would like to identify an appropriate assignment for all the arguments based on the assignment to the subset of arguments. 


More specifically, a partial function $\pi:\arguments\rightarrow[0,1]$ on $\arguments$ is called a \emph{partial probability assignment}. A probability function $P\in\allProb{\arguments}$ is \emph{$\pi$-compliant} if for every $\cA\in\dom{\pi}$ we have $\pi(\cA)=P(\cA)$. Let $\allProbComp{\AF}{\pi}\subseteq\allProb{\AF}$ be the set of all $\pi$-compliant probability functions. The question that arises is that given an abstract argumentation framework $\AFcomplete$ and a partial probability assignment $\pi$, how do we determine $P\in\allProb{\arguments}$ that is most compatible with both $\AF$ and $\pi$, i.\,e., which $P\in\allProb{\arguments}$ do we select as a meaningful representative? This question has also been addressed in similar ways for partial probabilistic information without argumentation, cf.\ e.\,g. \cite{Paris:1994}. There, the principle of \emph{maximum entropy} has been used to complete incomplete probabilistic information in probabilistic logics. 
As a first step, we investigate the properties of the sets of probability functions which are defined by our different axioms. An important requirement for applying maximum entropy approaches is that the probability function with maximum entropy is uniquely determined. A sufficient property to ensure this, is that the set under consideration is both \emph{convex} and \emph{closed}.\footnote{A set $X$ is \emph{convex} if for $x_{1},x_{2}\in X$ it also holds that $\delta x_{1}+(1-\delta)x_{2}\in X$ for every $\delta\in[0,1]$; a set $X$ is closed if for every converging sequence $x_{1},x_{2},\ldots$ with $x_{i}\in X$ ($i\in\mathbb{N}$) we have that $\lim_{i\rightarrow\infty}x_{i}\in X$} More generally, maximizing any strictly convex function over a convex set has a unique solution and also most interesting distance measures fall into this category. So if we are to identify probability functions that complete the missing assignments,  knowing that for specific sets of probability functions (i.\,e. those that satisfy specific axioms) that they are convex and closed, means that we may find an appropriate probability function.
\begin{proposition}\label{prop:incomplete}
	Let $\AFcomplete$ be an abstract argumentation framework.
	\begin{enumerate}
		\item The sets $\allProb{\AF}$, $\allProbT{\mathsf{COH}}{\AF}$, $\allProbT{\mathsf{TER}}{\AF}$, $\allProbT{\mathsf{NEU}}{\AF}$, $\allProbT{\mathsf{INV}}{\AF}$, $\allProbT{\mathsf{SFOU}}{\AF}$, $\allProbT{\mathsf{FOU}}{\AF}$, $\allProbT{\mathsf{OPT}}{\AF}$, $\allProbT{\mathsf{SOPT}}{\AF}$, $\allProbT{\mathsf{JUS}}{\AF}$, $\allProbT{\mathsf{MIN}}{\AF}$, $\allProbT{\mathsf{MAX}}{\AF}$ are convex and closed.
		\item The set $\allProbT{\mathsf{RAT}}{\AF}$ is not convex but closed.
	\end{enumerate}
	\begin{proof}~
	\begin{enumerate}
		\item Convexity and closure of $\allProb{\AF}$ has been shown in e.\,g. \cite{Paris:1994}, convexity and closure of $\allProbT{\mathsf{JUS}}{\AF}$ has been shown in \cite{Thimm:2012}. The convexity and closure of the other sets follow from the fact that they are defined using linear equations without strict inequality, cf.\ also \cite{Thimm:2012}.
		\item Let $\AFcomplete$ be given by $\arguments=\{\cA,\cB\}$ and $\attacks=(\cA,\cB)$. Consider $P_{1},P_{2}\in\allProbT{\mathsf{RAT}}{\AF}$ with
			\begin{align*}
				P_{1}(\cA) & = 1		& P_{1}(\cB) & = 0.4\\
				P_{2}(\cA) & = 0.4	& P_{2}(\cB) & = 0.8
			\end{align*}
			For the convex combination $P = 0.5 P_{1} + 0.5 P_{2}$ we obtain $P(\cA)=0.7$ and $P(\cB)=0.6$, i.\,e. $P\notin \allProbT{\mathsf{RAT}}{\AF}$. However, $\allProbT{\mathsf{RAT}}{\AF}$ is closed as for every converging sequence $P_{1},P_{2},\ldots$ with $P_{i}\in\allProbT{\mathsf{RAT}}{\AF}$ for all $i\in\mathbb{N}$ there is $N\in\mathbb{N}$ such that for all $\cA\attacks \cB$ either
			\begin{itemize}
				\item $P_{j}(\cA)\leq 0.5$ for all $j>N$; then $\lim_{i\rightarrow\infty} P_{j} (\cA)\leq 0.5$ as well and the condition of coherence is trivially satisfied, or
				\item $P_{j}(\cA)> 0.5$ and consequently $P_{j}(\cB)\leq 0.5$ for all $j>N$; then $\lim_{i\rightarrow\infty} P_{j} (\cB)\leq 0.5$ as well and the condition of coherence is satisfied in any case.
			\end{itemize}
	\end{enumerate}	
	\end{proof}
\end{proposition}
The above proposition suggests that most of our properties are suitable for convex optimization problems such as maximum entropy estimation. Moreover, the same is true for the notion of partial probability assignments.
\begin{proposition}\label{prop:convexity:partial}
	For every partial probability assignment $\pi$ the set $\allProbComp{\AF}{\pi}$ is convex and closed.
	\begin{proof}
		Let $P_{1},P_{2}\in\allProbComp{\AF}{\pi}$ and consider the convex combination $P=\delta P_{1}+(1-\delta)P_{2}$ for some $\delta\in[0,1]$. For every $\cA\in\dom{\pi}$ we have $P(\cA)=\delta P_{1}(\cA) + (1-\delta) P_{2}(\cA) = \delta \pi(\cA) + (1-\delta) \pi(\cA)=\pi(\cA)$ and therefore $P\in\allProbComp{\AF}{\pi}$. Closure of $\allProbComp{\AF}{\pi}$ is straightforward.
	\end{proof}
\end{proposition}
Let $t$ be any one of our properties which lead to a convex and closed set of probability functions (or any combination of those). If it is the case that there is at least one $\pi$-compliant $P$ in $\allProbT{t}{\AF}$ then (thanks to the convexity properties) we have that the intersection of $\allProbComp{\AF}{\pi}$ and $\allProbT{t}{\AF}$ is convex and closed as well, cf.\ \cite{Paris:1994}. In that case, we can select the probability function with maximal entropy within this intersection (which is uniquely defined). As for the rationale of this decision, several results from probability reasoning, as for example discussed in \cite{Paris:1994}, could be harnessed. 
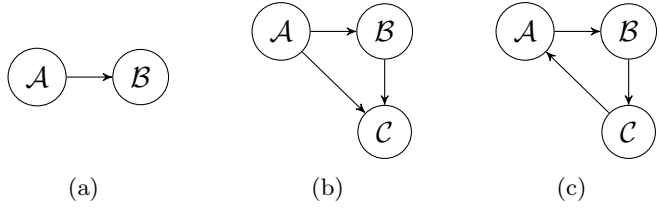
\begin{figure}
	\begin{center}
	\subfigure[]{
		\raisebox{7mm}{\begin{tikzpicture}[node distance=0.6cm]			
			\node[args](args1){$\cA$};
			\node[args, right=of args1](args2){$\cB$};		
			\path(args1) edge [->] (args2);
		\end{tikzpicture}}
		\label{fig:partialex:1}
	}\hspace*{0.7cm}
	\subfigure[]{
		\begin{tikzpicture}[node distance=0.6cm]			
			\node[args](args1){$\cA$};
			\node[args, right=of args1](args2){$\cB$};
			\node[args, below=of args2](args3){$\cC$};	
			\path(args1) edge [->] (args2);
			\path(args1) edge [->] (args3);
			\path(args2) edge [->] (args3);
		\end{tikzpicture}
		\label{fig:partialex:2}
	}\hspace*{0.7cm}
	\subfigure[]{
		\begin{tikzpicture}[node distance=0.6cm]			
			\node[args](args1){$\cA$};
			\node[args, right=of args1](args2){$\cB$};
			\node[args, below=of args2](args3){$\cC$};
			\path(args1) edge [->] (args2);
			\path(args3) edge [->] (args1);
			\path(args2) edge [->] (args3);
		\end{tikzpicture}
		\label{fig:partialex:3}
	}
	\caption{Argumentation frameworks for Example~\ref{ex:partialex}}
	\label{fig:partialex}
	\end{center}
\end{figure}
We continue with some examples to illustrate the definitions and to investigate some of our concerns in dealing with partial assignments.

\begin{example}\label{ex:partialex}
	For the argumentation framework depicted in Figure~\ref{fig:partialex}(a) consider $\pi_{1}$ with $\pi_{1}(\cA)=1$. Obviously, the most reasonable choice for a $\pi_{1}$-compliant $P\in\allProb{\AF}$ would be $P(\cA)=1$ and $P(\cB)=0$ (by obeying the property of involution). Furthermore, for $\pi_{2}(\cB)=0.3$ we would have $P(\cA)=0.7$ and $P(\cB)=0.3$ following the same rationale.
	
	For the argumentation framework depicted in Figure~\ref{fig:partialex}(b) consider $\pi_{3}$ with $\pi_{3}(\cC)=0.4$. A possible choice for $P$ would be $P(\cA)=0.6$, $P(\cB)=0.4$, and $P(\cC)=0.4$ (having thus a maximally committed function that is coherent). But note that the set $\allProbComp{\AF}{\pi}\cap \allProbT{\textsf{COH}}{\AF}$ does contain more than this single probability function. Furthermore, for $\pi_{4}$ with $\pi_{4}(\cB)=0.7$ and $\pi_{4}(\cC)=0.6$ one would only guess $P(\cA)\leq 0.3$ but due to the ``inconsistency'' of $\pi_{4}$ (violating the coherence condition), what is the best choice?
	
	For the argumentation framework depicted in Figure~\ref{fig:partialex}(c) consider $\pi_{5}$ with $\pi_{5}(\cA)=0.4$ and the following four selections $P_{1},P_{2},P_{3},P_{4}\in\allProb{\AF}$:
	\begin{align*}
		P_{1}(\cA) & = 0.4	&	P_{2}(\cA) & = 0.4		& P_{3}(\cA) & = 0.4		& P_{4}(\cA) & = 0.4\\
		P_{1}(\cB) & = 0.6	& 	P_{2}(\cB) & = 0.4			& P_{3}(\cB) & = 0.5	& P_{4}(\cB) & = 0.2\\
		P_{1}(\cC) & = 0.4	& P_{2}(\cC) & = 0.6			& P_{3}(\cC) & = 0.5				& P_{4}(\cC) & = 0.3
	\end{align*}
	All of the above probability functions are $\pi_{5}$-compliant and coherent. Function $P_{4}$ is not maximally committed and as such is perhaps not a good choice. Both $P_{1}$ and $P_{2}$ are ``extreme points of view'' and model some kind of probabilistic stable semantics. The function $P_{3}$ is as unbiased as possible but still ``reasonable'' as it models probabilistic grounded semantics. Note that $P_{3}$ is also the probability function with maximal entropy in $\allProbComp{\AF}{\pi}\cap \allProbT{\textsf{COH}}{\AF}$.
\end{example}
Given $\allProbComp{\AF}{\pi}$ and $\allProbT{t}{\AF}$, we can either select $P\in\allProbComp{\AF}{\pi}$ that is ``as close as possible to'' $\allProbT{t}{\AF}$ or $P\in \allProbT{t}{\AF}$ that is ``as close as possible to'' $\allProbComp{\AF}{\pi}$. In future work, we will investigate definitions for  ``as close as possible to'', and we will explore the pros and cons of each of these alternatives for selecting $P$.

\section{Discussion}\label{section:discussion}

In this paper, we have investigated the use of a probability function to represent belief in an argument. We use this to identify an epistemic labelling, and thereby an epistemic extension. 
The notion of an epistemic extension is very general as there are no constraints on the probability function in general. However,  we have considered various constraints on the probability function leading to  two views on using the probability functions, namely the standard view, and the non-standard view.  Many of the constraints on the probability function that we have investigated take into account aspects of the structure of the argument graph.
We applied our classification of properties of probabilistic argumentation to the problem of completing partial probability assignments. A first investigation leads us to believe that maximizing entropy within probability functions of a specific type gives appropriate results for this problem. In future work, we will investigate this issue in more depth.

The work in this paper contrasts with other research on introducing a probability assignment to arguments such as \cite{LON11}. There, a probability distribution over the subgraphs of the argument graph is introduced, and this can then be used to give a probability assignment for a set of arguments being an admissible set or extension of the argument graph. They assume independence between arguments which in general is not appropriate. To address this shortcoming, the set of spanning subgraphs can be used as a sample space, thereby obviating the need for an independence assumption between arguments \cite{Hun12b}.  This probability distribution over subgraphs is used to give a probability assignment to extensions. For a semantics $X$ (such as grounded, preferred, or stable), the probability that a set of arguments $\Gamma$ is an extension according to semantics $X$, denoted $P_X(\Gamma)$, is the sum of the probability assigned to the subgraphs for which $\Gamma$ is an extension according to semantics $X$. The uncertainty that is being handled in  \cite{LON11,Hun12b} is about the structure of the graph, and it is therefore a different kind of uncertainty model to the being addressed by this paper. Similarly, the work \cite{Janssen:2008} extends abstract argumentation by allowing the attack relation to be a \emph{fuzzy relation}. 

\bibliographystyle{alpha}
\bibliography{probarg_ee}


\end{document}